\documentclass[a4paper,10pt]{scrartcl}

\usepackage{amsmath,amssymb,amsfonts}
\usepackage{ifthen}
\usepackage{hyperref}
\usepackage[amsmath,thmmarks,hyperref]{ntheorem}
\usepackage[boxed]{algorithm2e}
\usepackage[english]{babel}
\usepackage{todonotes}
\usepackage[affil-it]{authblk}
\usepackage{graphicx,import}

\theoremstyle{plain}
\theoremseparator{.}
\newtheorem{theorem}{Theorem}
\newtheorem{lemma}[theorem]{Lemma}
\newtheorem{corollary}[theorem]{Corollary}
\newtheorem{problem}{Problem}
\theorembodyfont{}
\theoremsymbol{\ensuremath{{}_\blacksquare}}
\theoremheaderfont{\itshape}
\newtheorem{remark}[theorem]{Remark}

\theoremsymbol{\ensuremath{\square}}
\newtheorem*{proof}{Proof}

\DeclareMathOperator*{\argmin}{arg\,min}
\newcommand{\field}[1]{\mathbb{#1}}
\newcommand{\R}{\field{R}}

\DeclareMathOperator{\Diff}{Diff}
\newcommand{\ind}{{k,l;i,j}}
\newcommand{\vC}{\mathbf{C}}
\newcommand{\vtheta}{\boldsymbol{\theta}}

\date{\today}

\author{Martin Bauer\thanks{Faculty of Mathematics, University of
    Vienna.\\ E-mail:
    \href{mailto:bauer.martin@univie.ac.at}{bauer.martin@univie.ac.at}
    \\ MB was supported by FWF project P24625.}, Markus
  Eslitzbichler\thanks{Department of Mathematical Sciences, NTNU
    Trondheim.\\ E-mail:
    \href{mailto:markus.eslitzbichler@math.ntnu.no}{markus.eslitzbichler@math.ntnu.no}},
  Markus Grasmair\thanks{Department of Mathematical Sciences, NTNU
    Trondheim.\\ E-mail: \href{mailto:markus.grasmair@math.ntnu.no}{markus.grasmair@math.ntnu.no}}.}
\title{Landmark-Guided Elastic Shape Analysis of Human Character Motions}

\begin{document}

\maketitle
\begin{abstract}
{\bf Abstract.} 
Motions of virtual characters in movies or video games are typically
generated by recording actors using motion capturing methods.
Animations generated this way often need postprocessing, such as improving the periodicity of cyclic animations or generating entirely new motions by interpolation of existing ones.
Furthermore, search and classification of recorded motions becomes more and more important as the amount of recorded motion data grows.

In this paper, we will apply methods from shape analysis to the processing of animations.
More precisely, we will use the by now classical elastic metric model
used in shape matching, and extend it by incorporating additional
inexact feature point information, which leads to an improved temporal
alignment of different animations.


\end{abstract}

\section{Introduction}

\paragraph{Animations.}
Virtual characters in movie and TV special effects or video games are most commonly animated using \emph{skeletal animation}, where a character's motion is described in terms of joint-angles in an approximation of a human skeleton.
\emph{Motion capturing} is a typical way to generate such animations, whereby an actor or stuntman performs the requested motion while being recorded from multiple angles.
From this, the underlying skeletal pose can then be reconstructed.
While motion capturing can produce very life-like animations, it has a number of drawbacks, some of which can be addressed algorithmically \cite{pejsa_state_2010, bruderlin_motion_1995}.

In \cite{eslitzbichler2014}, methods from shape analysis were applied to a number of problems in computer animation: periodicity of animations, interpolation between animations and motion recognition.
A central component in this is the modelling of curve reparametrizations as elements of the diffeomorphism group on the circle or the unit interval, $\operatorname{Diff}(S^1)$ and $\operatorname{Diff}([0,1])$ respectively.
For animations, such reparametrizations can be used to align them in time, similar to the concept of \emph{timewarp curves} in \cite{kovar_flexible_2003, kovar_automated_2004}.
Note that the skeletal animation approach is very different in concept and applications from silhouette-based representations of motions, as for example in \cite{abdelkader_silhouette-based_2011}.

In this paper we will demonstrate how feature point information can be
incorporated into shape matching techniques, and how this can be
applied in the context of computer animations in order to improve the
temporal alignment of related actions.

\paragraph{Shape matching.}
The field of shape analysis concerns itself with the study and classification of similarities and dissimilarities within certain classes of shapes. 
In order to achieve this objective, a variety of different methods has
been developed, each tailored to the actual class of shapes under consideration.
A particularly important example for a shape space, which is also the
main focus of this article, is the space of unparametrized curves,\footnote{Closed curves can be used, for example, to represent outlines of objects in recognition applications.} and,
in recent years, Riemannian methods for shape spaces of curves have been deeply explored both theoretically \cite{Younes1998,Mio2004,Mumford2006,Michor2007,Fuchs2009,Shah2008,Mennucci2008,Shah2013,Bruveris2014,Bauer2014_preprint} as well as from an application oriented point of view 
\cite{Mio2007,Samir2012,LKS2014,SKKS2014}. See \cite{Bauer2014a} for an overview of these topics.

Although the main contribution of this article is not limited to
a specific Riemannian metric on the shape space of unparametrized curves,
we will focus for simplicity on a particular one that is related to the so called \emph{Square Root Velocity} (SRV) framework \cite{Jermyn2011,Bauer2014b}.
This metric is given as the pullback of the $L^2$--metric via a transformation called the SRVT. This allows for the development of extremely efficient numerical methods and, as a consequence,
it has been successfully used in a series of applications \cite{XKS2014,eslitzbichler2014}. See also \cite{Michor2008a,Bauer2014c} for other reparametrization invariant metrics that can be represented in a similar way.

Our main addition to this standard setting will be the incorporation
of point correspondences between two curves one wants to match. In the
particular application of animation processing, these correspondences
will describe similar poses at different points in time in the two animations. We assume here
that the point correspondences are manually entered pieces of
information that are possibly error-prone.
Thus we are interested in an exact matching of the unparametrized curves, but
only an inexact matching of the point correspondences. 
Our approach differs from the one presented in \cite{wei_liu_riemannian_2011}, where feature information is included in the form of  auxiliary functions that are combined with the geometric curve to form higher dimensional curves that can then be matched by the usual methods.
Instead, we augment the matching energy functional with an extra feature matching term, leaving the elastic matching terms unchanged.
In comparison to \cite{kurtek_landmark-guided_2013}, which focused on matching three dimensional surfaces, we allow inexact matching of landmarks in our model.
We will discuss in this article both the theoretical framework of
shape matching with feature points and computational aspects.

\paragraph{Overview}
In Sect.~\ref{Sect:GeneralFramework} we give an overview of the general framework for shape
matching with additional feature point information 
and will then apply this framework to one particular choice of metric:
the metric that is induced by the SRV transformation. We then present
the basics of two algorithms that can be used to determine the optimal
parametrization to match two curves:
In Sect.~\ref{gradient} we derive all the formulas necessary to
find the parametrization using a gradient descent approach. In
addition, in Sect.~\ref{dynamic}, we discuss the application of a
dynamic programming based algorithm.
In Sect.~\ref{numerics} we apply the previously presented framework to
process human motion data.
Moreover, we show some examples of an application to
two-dimensional curves, where the workings of the method can be better
visualized.

\section{The general framework}\label{Sect:GeneralFramework}
\paragraph{Problem formulation.}
We will start this article by formulating our main problem:
\begin{problem}\label{problemformulation}
Given two unparametrized curves $[c_0]$ and $[c_1]$ and a number of $n$ point correspondences between the curves -- i.e., points $C^i_0$ and $C_1^i$ that should be matched onto each other --
we want to find an optimal deformation between the two curves, that also respects the alignment of the feature points. 
\end{problem}

In order to achieve this goal, we will first choose representatives of the shapes $[c_0]$ and $[c_1]$, i.e., parametrized curves $c_0$ and $c_1$. Each representative $c_j$
determines parameter values $\theta_j^i$ that correspond to the feature points $C_j^i$ in the sense that
$$c_j(\theta_j^i)=C^i_j.$$
Thus we can represent any shape $[c_0]$ with additional feature point information as a tuple $(c_0,\vtheta_0)$, with $c_0\in \operatorname{Imm}(S^1,\mathbb R^d)$ and $\vtheta_0 = (\theta_0^i)_i\in (S^1)^n$.
One way to tackle Problem \ref{problemformulation} is then to construct a similarity measure on the product space $\operatorname{Imm}(S^1,\mathbb R^d)\times (S^1)^n$ that has  certain 
invariance properties with respect to the action of the diffeomorphism group.  

To do this mathematically rigorously we will need some results and definitions from infinite dimensional Riemannian geometry first.

\subsection{The manifold of parametrized curves}
In this article we will consider the space of regular curves from a parameter space $M$ into some -- possibly high dimensional -- $\R^d$:
\begin{equation}
\operatorname{Imm}(M,\R^d):=\left\{c\in C^{\infty}(M,\R^d): |c'|>0 \right\}.
\end{equation}
Here $M$ stands for the circle $S^1$ in the case of closed curves and for the interval $[0,2\pi]$ for open curves.

The space $\operatorname{Imm}(M,\R^d)$ is a smooth infinite dimensional Fr\'echet manifold 
with tangent space  $T_c\operatorname{Imm}(M,\R^d)$, the set of all vector fields 
along the curve $c$. Using the trivialization of $\R^d$ we can identify the tangent space with
\begin{equation}
T_c\operatorname{Imm}(M,\R^d):=\left\{h\in C^{\infty}(M,\R^d)\right\}.
\end{equation}

On the manifold $\operatorname{Imm}(M,\R^d)$ we  consider  reparametrization invariant metrics $G_c$, i.e.,
Riemannian metrics that satisfy
\begin{equation}
G_c(h,k)=G_{c\circ\varphi}(h\circ\varphi,k\circ\varphi)\quad\text{for all}\quad \varphi \in \operatorname{Diff}(M).
\end{equation}
Here, $\operatorname{Diff}(M)$ denotes the group of orientation preserving diffeomorphisms of $M$, which consists of all reparametrizations of
the curves under consideration.

The most prominent example of such a metric is the elastic metric $G^{a,b}$ that is defined  by
\begin{equation}\label{elasticmetric}
G^{a,b}_c(h,k)=\int_{M}a^2 |(D_s h)^\bot| |(D_s k)^\bot|+b^2|(D_s h)^\top| |(D_s k)^\top| ds.
\end{equation}
Here $a^2$, $b^2$ are positive constants, $D_s=\frac{1}{|c'|}\partial_{\theta}$ denotes the arc length derivative, $ds=|c'|d\theta$ is arc length integration, 
$v=D_s c$ is the unit length tangent vector, $|(D_s h)^\top|=\langle D_sh, v \rangle$ 
denotes the tangential component of $D_s h$ and $|(D_s h)^\bot| = D_sh - |(D_s h)^\top|v $ the normal component. 
Note that constant vector fields $h$ are in the kernel of $G^{a,b}$, thus \eqref{elasticmetric} defines only a metric on the manifold of immersions modulo translations. 

Other metrics that have been introduced include higher order Sobolev metrics, i.e., metrics of the form
\begin{equation}\label{higherSobmetric}
G^{l}_c(h,k)=\int_{M}\left(\sum_{j} a_j \langle D^j_s h, D^j_s k \rangle \right) ds,
\end{equation}
with coefficients $a_j$ possibly depending on the foot point $c$. Depending on the order of the metric $G$, local and global well-posedness of the geodesic equation have been shown and the metric completions of the corresponding spaces 
have been studied \cite{Bruveris2014}.

In the experimental part of this article we will focus on one particular member of this family, namely the elastic metric that corresponds to the parameters $a=1$, $b=\frac12$. 
This 
metric has the advantage that it has a very helpful representation as a pullback of the flat $L^2$--metric. To see this we introduce the so-called \emph{square root velocity transform} (or short \emph{SRVT}):
\begin{align}
R:\begin{cases}
  \operatorname{Imm}(M,\mathbb R^d)&\mapsto C^{\infty}(M,\mathbb R^d\setminus\{0\}),\\
  c&\rightarrow \dfrac{c'}{\sqrt{|c'|}}.
  \end{cases}
\end{align}
The SRVT, when regarded for curves modulo translations,
has an inverse, which is given by:
\begin{align}
R^{-1}:\begin{cases}
   C^{\infty}(M,\mathbb R^d\setminus\{0\})&\mapsto \operatorname{Imm}(M,\mathbb R^d), \\
  q&{\displaystyle\rightarrow \int_0^\tau |q|q \,d\theta}.
  \end{cases}
\end{align}

On $C^{\infty}(M,\mathbb R^d)$ we can consider the flat $L^2$--metric. In \cite{Jermyn2011} it has been shown, that the pullback via $R$ of the 
$L^2$--metric is exactly the elastic metric \eqref{elasticmetric}. 
The situation for open curves is particularly easy, as one has then explicit formulas for geodesics and geodesic distance:
\begin{theorem}\label{SRVT-open}
The image of the $R$-map of the manifold of open curves is an open subset of $C^{\infty}([0,2\pi],\mathbb R^d)$:
\begin{equation}
\operatorname{Im}(R) = \left\{q\in  C^{\infty}([0,2\pi],\mathbb R^d): |q|\neq 0 \right\} .
\end{equation}
Two open curves $c_0$, $c_1$ can be connected by a geodesic if and only
if there exist no $\theta \in [0,2\pi]$ and no $\lambda > 0$ such that $c_0'(\theta) = -\lambda c_1'(\theta)$.
In this case, the unique geodesic connecting them is given by
\begin{equation}
c(t,\theta) =  R^{-1}\left((1-t)R(c_0)+tR(c_1)\right).
\end{equation}
Moreover, the induced geodesic distance on $\operatorname{Imm}([0,2\pi],\R^d)$ is given by
\begin{equation}
d(c_0,c_1) =  \sqrt{\int_0^{2\pi} \|R(c_0)-R(c_1)\|^2_{\mathbb R^d}d\theta}\,.
\end{equation}
\end{theorem}
For a proof of this theorem see \cite{Bauer2014b}. The situation for closed curves is less explicit.
We have the following characterization of the image of the SRVT, which will build the fundament of our algorithms:

\begin{theorem}\label{SRVT-closed}
The image $R(\operatorname{Imm}(S^1,\mathbb R^d))$ of the manifold of closed curves under the SRVT-transform is a codimension $d$ submanifold of the flat space $C^{\infty}(S^1, \mathbb R^d)$.
A basis of the orthogonal complement $\left(T_qR(\operatorname{Imm}(S^1,\mathbb R^d)\right)^\perp$ is given by the $d$ vectors
\begin{align}\label{basis}
U_i(q)=\frac{1}{|q|}\,\left( q_i q_1, \dots, q_i^2+|q|^2, \dots, q_i q_d  \right).
\end{align} 
\end{theorem}
Using this basis, efficient numerical methods for calculating geodesics between closed curves have been developed, see \cite{Jermyn2011,Bauer2014b}.

\subsection{The matching functional on the space of parametrized curves}
To define our similarity measure on the product space  $\operatorname{Imm}(M,\R^d)\times M^n$, we will first 
introduce an energy functional that is defined for arbitrary paths in $\operatorname{Imm}(M,\R^d)$. 
We will then define the similarity measure as the minimal energy over all paths with given boundary shapes $[c_0]$ and $[c_1]$. The important features of the similarity measure on the space of unparametrized curves will be:
\begin{itemize}
 \item[(a)] The similarity measure does not depend on the choice of representatives $(c_i,\vtheta_i)$ of the observed shapes $([c_i],\vC^i)$.
 \item[(b)] The optimal deformation is guided by both the shape of the boundary curves and by the feature point information.
 \item [(c)] The similarity measure forces an exact matching of the unparametrized curves, but only an inexact matching of the feature point information.
\end{itemize}

Given a curve $\hat{c} \in \operatorname{Imm}(M,\R^d)$, we denote in
the following by $\hat{c}(\vtheta)$ the whole vector of points
$\hat{c}(\theta^i) \in \R^d$, $1 \le i \le n$.

For parameter values $\vtheta_0 = (\theta_0^i)_i\in M^n$ and feature points $\vC_1 = (C_1^i)_i\in (\mathbb R^d)^n$
we define the energy functional for a given path 
$c\colon [0,1]\rightarrow \operatorname{Imm}(M,\mathbb R^d)$ as:
\begin{equation}\label{energy_funct}
\mathcal E(\vtheta_0,\vC_1)(c)=\int_0^1 G_c(c_t,c_t) dt + \lambda \operatorname{FM}\left(c(1,\vtheta_0),\vC_1\right)\,.
\end{equation}
Here $G_c(\cdot,\cdot)$ is any reparametrization invariant metric on $\operatorname{Imm}(M,\mathbb R^d)$ and $\operatorname{FM}$ denotes some similarity measure 
on $\R^{d\times n}$.
The only conditions on $\operatorname{FM}$ we impose at the moment are
that $\operatorname{FM}$ is lower semi-continuous and
$\operatorname{FM}(\vC,\vC) = 0$. 
The first condition is necessary for the subsequent energy minimization,
while the second
condition implies that constant paths actually have zero energy.
When we discuss later the actual computation of
energy minimizing paths, we will introduce further conditions that make their
numerical approximation possible.

\begin{lemma}\label{lem:invariance}
The energy functional \eqref{energy_funct} satisfies the invariance property 
\begin{align}
 \mathcal E(\varphi^{-1}(\vtheta_0),\vC_1)(c\circ\varphi)=\mathcal E(\vtheta_0,\vC_1)(c).
\end{align}
\end{lemma}
\begin{remark}
The meaning of this invariance property will become clear in Sect. \ref{sim_unparametrized}, where we will consider the action of the diffeomorphism group on the 
quotient space $\operatorname{Imm}(M,\mathbb R^d)\times M^n$.
\end{remark}

\begin{proof}
Using the reparametrization invariance of the metric $G_c$ we calculate
\begin{align*}
 \mathcal E(\varphi^{-1}(\vtheta_0),\vC_1)(c\circ\varphi)&= 
 \int_0^1 G_{c\circ\varphi}((c\circ\varphi)_t,(c\circ\varphi)_t) dt + \lambda \operatorname{FM}\left((c\circ\varphi)(1,\varphi^{-1}(\vtheta_0)),\vC_1\right)\\
 &=\int_0^1 G_c(c_t,c_t) dt + \lambda \operatorname{FM}\left(c(1,\vtheta_0),\vC_1\right)=\mathcal E(\vtheta_0,\vC_1)(c).
\end{align*}
\end{proof}
Using this energy functional we define our similarity measure on the product space $\operatorname{Imm}(M,\mathbb R^d)\times M^n$ of parametrized curves with feature points as
\begin{equation}\label{sim_parametrized}
\boxed{
d_P\left((c_0,\vtheta_0),(c_1,\vtheta_1)\right):= \underset{c:[0,1]\rightarrow \operatorname{Imm}}{\operatorname{inf}} \mathcal E(\vtheta_0,c_1(\vtheta_1))(c)}
\end{equation}
where the infimum is taken over all paths $c$ in $\operatorname{Imm}(M,\R^d)$ that satisfy
\begin{equation}
c(0,\cdot)=c_0 \text{ and } c(1,\cdot)=c_1\,.
\end{equation}
\begin{remark}
We do not call the similarity measure a distance, since it is not symmetric in general. However, it would be straightforward to construct a symmetric version of this. This will be described in Sect. \ref{sect:symmetric}.
\end{remark}

Because we fix the endpoint $c(1,\cdot)=c_1$, we can write the similarity
measure $d_P$ as
\[
d_P((c_0,\vtheta_0),(c_1,\vtheta_1))
= \inf_{c:[0,1]\to\operatorname{Imm}} \Bigl[\int_0^1 G_c(c_t,c_t)\,dt\Bigr]
+ \lambda\operatorname{FM}(c_1(\vtheta_0),c_1(\vtheta_1)).
\]
That is, we only minimize the first term of the energy functional,
and we do not allow to change the value of the second term at all. 
The meaning of the second term, will become clear when we consider it on the shape space of unparametrized curves.

\subsection{The similarity measure on the shape space of unparametrized, feature curves.}\label{sim_unparametrized}
In this section we want to use the previously defined similarity measure on parametrized curves to induce a similarity measure on the shape space of unparametrized curves with feature point information.
Therefore we have to determine the induced action of the diffeomorphism group on the product space $\operatorname{Imm}(M,\R^d)\times M^n$. 
On the first factor $\operatorname{Imm}(M,\R^d)$ it is simply given by composition from the right.
To compute the action on the second factor $M^n$ we need to compute the effect of a reparametrization on the feature points. We have:
\begin{align}
C_0^i=c_0(\theta_0^i)=c(\varphi(\varphi^{-1}(\theta_0^i)))=(c\circ\varphi)(\varphi^{-1}(\theta_0^i)). 
\end{align}
Thus the induced action on the product space is given by
\begin{equation}
(c_0,\vtheta_0)\circ\varphi=(c_0\circ\varphi,\varphi^{-1}(\vtheta_0)). 
\end{equation}
Using the invariance of our similarity measure -- cf. Lemma~\ref{lem:invariance} -- we obtain the following result:
\begin{theorem}\label{thm:similarity_measure}
The similarity measure \eqref{sim_parametrized} on $\operatorname{Imm}(M,\R^d)\times M^n$ induces a similarity measure on 
the shape space of unparametrized curves with additional feature point information. The induced functional is given by:
\begin{equation}\label{distance_feature_curves}
d\left(([c_0],\vC_0),([c_1],\vC_1)\right) := \underset{\varphi\in \operatorname{Diff}(M)}{\inf} d_P\left((c_0,\vtheta_0),(c_1\circ\varphi,\varphi^{-1}(\vtheta_1)\right).
\end{equation}
Here $(c_j,\vtheta_j)$  are arbitrary representatives of the shapes $([c_j],\vC_j)$.
\end{theorem}
\begin{remark}
Note, that the energy functional will force an exact matching of the unparametrized curves, but only an inexact matching of the feature points. The reason for this is that we assume the feature points 
to be an additional manually entered information that is possibly error-prone.
\end{remark}

\begin{proof}
 We need to show that $d$ does not depend on the actual choice of representatives $c_0$ and $c_1$. 
 Any other representatives of $[c_i]$ can be written as $c_i\circ\varphi$ for some diffeomorphism $\varphi$. Since we are minimizing over
 all possible reparametrizations of $c_1$, the functional clearly does not depend on the choice of the representative $c_1$. It remains to verify
 the independence of reparametrizations of $c_0$. Therefore we calculate
 \begin{multline*}
 \underset{\varphi\in \operatorname{Diff}(M)}{\inf}  d_P\left((c_0 \circ \psi ,\psi^{-1}(\vtheta_0)),(c_1\circ\varphi,\varphi^{-1}(\vtheta_1)\right)\\=
\underset{\varphi\in \operatorname{Diff}(M)}{\inf} \left(\underset{c:[0,1]\rightarrow \operatorname{Imm}}{\operatorname{inf}} \mathcal E\left(\psi^{-1}(\vtheta_0),c_1(\vtheta_1)\right)(c)\right)\,,
 \end{multline*}
where the infimum is taken over all paths $c$ that satisfy the boundary conditions
\begin{equation*}
 c(0,\cdot)=c_0\circ\psi, \quad c(1,\cdot) = c_1\circ\varphi\,.  
\end{equation*}
Using the invariance property of the functional -- cf.~Lemma \ref{lem:invariance} -- we can rewrite this as
\begin{align*}
& \underset{\varphi\in \operatorname{Diff}(M)}{\inf}  d_P\left((c_0 \circ \psi ,\psi^{-1}(\vtheta_0)),(c_1\circ\varphi,\varphi^{-1}(\vtheta_1)\right)\\&\qquad\qquad=
\underset{\varphi\in \operatorname{Diff}(M)}{\inf} \left(\underset{c:[0,1]\rightarrow \operatorname{Imm}}{\operatorname{inf}} \mathcal E\left(\vtheta_0,c_1(\vtheta_1)\right)(c\circ\psi^{-1})\right)\\&\qquad\qquad=
\underset{\varphi\in \operatorname{Diff}(M)}{\inf} \left(\underset{\tilde c:[0,1]\rightarrow \operatorname{Imm}}{\operatorname{inf}} \mathcal E\left(\vtheta_0,c_1(\vtheta_1)\right)(\tilde c)\right)
 \end{align*}
such that 
\begin{equation*}
 \tilde c(0,\cdot)=c_0\circ\psi\circ\psi^{-1}=c_0, \quad c(1,\cdot) = c_1\circ\varphi, 
\end{equation*}
which concludes the proof.
\end{proof}

We note that the similarity measure $d$ can also be written as
\begin{multline*}
d(([c_0],\vC_0),([c_1],\vC_1))\\
= \inf_{\varphi\in\operatorname{Diff}(M)}\biggl(\inf_{c:[0,1]\to\operatorname{Imm}} \biggl[\int_0^1 G_c(c_t,c_t)\,dt\biggr]
+ \lambda\operatorname{FM}(c_1\circ\varphi(\vtheta_0),\vC_1)\biggr),
\end{multline*}
where the inner infimum is taken over all paths $c$ that satisfy the conditions
$c(0,\cdot) = c_0$ and \mbox{$c(1,\cdot) = c_1 \circ \varphi$}.

\begin{remark}
Due to the invariance with respect to the reparametrization group, all the  metrics $G_c$ descend to the shape space 
of unparametrized curves, i.e., they induce a metric on the quotient space 
$\mathcal S:=\operatorname{Imm}(M,\R^d)/\operatorname{Diff}(M)$ such that the projection $$\pi: \operatorname{Imm}(M,\R^d)\rightarrow \operatorname{Imm}(M,\R^d)/\operatorname{Diff}(M)$$
is a Riemannian submersion. For a detailed discussion of this topic we refer to the article \cite{Michor2007}.  For $\lambda=0$ -- i.e., no feature point matching -- the similarity measure \eqref{distance_feature_curves}
is then given by the induced geodesic distance on the quotient space. If we assume existence of a minimizer, it would be given by a horizontal geodesic on the top space, the manifold of parametrized curves.
For $\lambda >0$ minimizers of \eqref{distance_feature_curves} will still be geodesics on $\operatorname{Imm}(M,\R^d)$, however their initial velocity will in general not be horizontal anymore. The induced curve on the quotient space 
$\operatorname{Imm}(M,\R^d)/\operatorname{Diff}(M)$ will thus not be a geodesic. In \cite{AlKrLoMi2003} such curves have been called ballistic curves.
\end{remark}

\begin{remark}
  In certain applications, there exist natural \emph{reference parametrizations}
  $c_{i,\textrm{ref}} \in [c_i]$ of the curves one is interested in. For instance, if one deals with
  skeletal animations (for details see Section~\ref{numerics} below), the curves are mappings
  from a time interval into the so-called joint space.
  In this case, the reference parametrization of a given animation uses a uniform
  frame rate, and reparametrizations correspond to local speed-ups or slow-downs.
  In such a setting, it makes sense to define the feature matching term based
  on the similarity of parameter values rather than the points on the curve.
  Given some distance measure $\operatorname{\widehat{FM}}$ on $M^n$,
  this can be achieved in our setting by defining
  \[
  \operatorname{FM}(\vC_1,\vC_2) :=
  \begin{cases}
    \operatorname{\widehat{FM}}(c_{0,\textrm{ref}}^{-1}(\vC_0),c_{1,\textrm{ref}}^{-1}(\vC_1))
    & \text{ if } \vC_i \in c_{i,\textrm{ref}}(M)^n,\\
    +\infty & \text{ else.}
  \end{cases}
  \]
  Then the similarity measure $d$ can be written as
  \begin{multline*}
    d(([c_0],\vC_0),([c_1],\vC_1))\\
    = \inf_{\varphi\in\operatorname{Diff}(M)}\biggl(\inf_{c:[0,1]\to\operatorname{Imm}} \biggl[\int_0^1 G_c(c_t,c_t)\,dt\biggr]
    + \lambda\operatorname{\widehat{FM}}(\varphi(\vtheta_{0,\textrm{ref}}),\vtheta_{1,\textrm{ref}})\biggr),
  \end{multline*}
  where we consider in the inner infimum only paths satisfying
  $c(0,\cdot) = c_{0,\textrm{ref}}$ and $c(1,\cdot) = c_{1,\textrm{ref}}\circ\varphi$.
\end{remark}

\subsection{Symmetrization of the feature matching term}\label{sect:symmetric}

With the definition in~\eqref{energy_funct}, the energy is not symmetric
with respect to the two shapes, because the feature points are treated differently.
It is, however, straightforward to symmetrize the energy functional by defining
\[
\mathcal{E}_{\rm sym}(\vtheta_0,\vC_0,\vtheta_1,\vC_1)(c) 
= \int_0^1 G_c(c_t,c_t)dt + \lambda\bigl(\operatorname{FM}(c(1,\vtheta_0),\vC_1)
+ \operatorname{FM}(c(0,\vtheta_1),\vC_0)\bigr).
\]
This energy functional satisfies the invariance property
\[
\mathcal{E}(\varphi^{-1}(\vtheta_0),\vC_0,\varphi^{-1}(\vtheta_1),\vC_1)(c\circ\varphi)
= \mathcal{E}(\vtheta_0,\vC_0,\vtheta_1,\vC_1)(c)
\]
for any diffeomorphism $\varphi$ of $M$.
From this energy functional we obtain a distance on the space of
parametrized curves with feature points,
\[
d_{P,\textrm{sym}}\bigl((c_0,\vtheta_0),(c_1,\vtheta_1)\bigr)
:= \inf_{c:[0,1]\to\operatorname{Imm}} \mathcal{E}_{\rm sym}\bigl(\vtheta_0,c_0(\vtheta_0),\vtheta_1,c_1(\vtheta_1)\bigr)(c),
\]
where the infimum is taken over all paths $c$ such that $c(0,\cdot)=c_0$ and $c(1,\cdot) = c_1$.
Now, the invariance property of $\mathcal{E}$ implies that
\[
d_{P,\textrm{sym}}\bigl((c_0\circ\psi,\psi^{-1}(\vtheta_0)),(c_1\circ\psi,\psi^{-1}(\vtheta_1))\bigr)
= d_{P,\textrm{sym}}\bigl((c_0,\vtheta_0),(c_1,\vtheta_1)\bigr),
\]
whenever $\psi$ is a diffeomorphism of $M$.
This allows, similar to Theorem~\ref{thm:similarity_measure}, to define a 
symmetric similarity measure on the shape space of unparametrized curves with feature points by
\begin{equation}\label{eq:d_symm}
d(([c_0],\vC_0)),([c_1],\vC_1))
= \inf_{\varphi\in\operatorname{Diff}(M)}
d\bigl((c_0,c_0^{-1}(\vC_0)),(c_1\circ\varphi,\varphi^{-1}\circ c_1^{-1}(\vC_1))\bigr).
\end{equation}

The problem of this similarity measure is that the computation
of the infimum in~\eqref{eq:d_symm} requires the evaluation
of terms of the form $\operatorname{FM}(\varphi^{-1}\circ c_1^{-1}(\vC_1),\vC_0)$,
which involve the inverse of the diffeomorphism $\varphi$.
In particular for derivative based optimization methods like
gradient descent, this poses problems, as they would require in addition
to $\varphi^{-1}$ also its derivative.
For this reason we have used only the non-symmetric similarity measure
in all the computational examples below.
We note, however, that discretizations of the symmetric term can, in certain cases,
be minimized efficiently with an approach based on dynamic programming
(see Sect.~\ref{dynamic}).

\section{Matching feature curves with the elastic metric.}
In the following, we will study one particular choice for both the Riemannian metric and the feature matching term.
Our choice of the Riemannian metric $G$ on $\operatorname{Imm}(M,\mathbb R^d)$ is the elastic metric with coefficients $a=1, \ b=\frac12$, see  equation \eqref{elasticmetric}.
This is particularly beneficial if we work on open curves. In this case, we have an explicit formula
for the induced geodesic distance of the metric $G$, cf.~Theorem \ref{SRVT-open}. 
For the feature point matching we use the squared $\ell^2$--norm on the parameter
space $M$ with respect to some reference curves $c_{i,\textrm{ref}} =: c_i$.
We will discuss different choices for the feature matching term below in Remarks \ref{rem:otherFM} and \ref{rem:otherFM2}.

As a direct consequence, we obtain the following formula for the  matching functional for feature curves:

\begin{corollary}
Using the elastic metric $G$ with coefficients $a=1, b=\frac12$ and the $\ell^2$--norm error term the similarity measure  \eqref{distance_feature_curves} on the set of open feature curves reads as:
\begin{multline}
d\left(([c_0],\vC_0),([c_1], \vC_1)\right) \\= \underset{\varphi\in \operatorname{Diff}(S^1)}{\inf} \left(\int_0^{2\pi} \|\frac{c_0'}{\sqrt{|c_0'|}}-\sqrt{\varphi'}\frac{c_1'\circ\varphi }{\sqrt{|c_1'|} \circ \varphi}  \|^2_{\mathbb R^d}d\theta 
+ \lambda \sum_{i=1}^n |\varphi(\theta_0^i)-\theta_1^i|^2\right)\,.
\end{multline}
Here, $(c_j,\vtheta_j)$  are arbitrary representatives of the shapes $([c_j], \vC_j)$.
\end{corollary}
This observation yields the following strategy for solving the feature curves matching problem:
\begin{itemize}
 \item Minimize 
  \begin{align}\label{matching_funct3}
\mathcal E_1^{\operatorname{op}}(\varphi):= \int_0^{2\pi} \|\frac{c_0'}{\sqrt{|c_0'|}}-\sqrt{\varphi'}\frac{c_1'\circ\varphi }{\sqrt{|c_1'|} \circ \varphi}  \|^2_{\mathbb R^d}d\theta 
+ \lambda \sum_{i=1}^n |\varphi(\theta_0^i)-\theta_1^i|^2\,, 
\end{align}
over $\varphi\in \Diff([0,2\pi])$.
 \item Calculate the geodesic connecting $c_0$ to $c_1\circ \varphi$ using the explicit formula from Theorem \ref{SRVT-open}.
\end{itemize}
\begin{remark}
Note, that the minimizer of \eqref{matching_funct3} will in general not be a diffeomorphism, but will only have a non-negative derivative, cf. \cite{Michor2008a}
To guarantee existence of the minimizer in the diffeomorphism group one would need to use a stronger metric on $\operatorname{Imm}(M,\mathbb R^d)$ see \cite{Bruveris2014}.
\end{remark}
For closed curves the situation is more complicated, since there is no explicit formula for the geodesic distance. Thus the matching functional does not simplify and \eqref{matching_funct3}
for closed curves reads as:
 \begin{equation}\label{matching_funct4}
\mathcal E_1^{\operatorname{cl}}(\varphi):=  \operatorname{dist}(c_0,c_1 \circ \varphi)^2  + \lambda \sum_{i=1}^n |\varphi(\theta_0^i)-\theta_1^i|^2.
\end{equation}
However, due to the explicit characterization of the image of the SRVT, there are fast and efficient ways to numerically calculate the geodesic distance, cf. \cite{Jermyn2011,Bauer2014b}. 

In the following we will present two methods to minimize these functionals: a dynamic programming approach and a gradient descent algorithm. 

\subsection{A gradient descent approach}\label{gradient}
The integral component of the gradient descent algorithm -- the variation of the energy functional \eqref{matching_funct3} -- will be derived in the following lemma. To simplify the exposition we introduce the notation
\begin{equation}
q\star \varphi := \sqrt{\varphi'} (q\circ\varphi), 
\end{equation}
for the action of the diffeomorphism group on the space of SRV-transformed functions. Then we have:
\begin{lemma}
The variation of $\mathcal E^{\operatorname{op}}$ in direction $\delta \varphi$ is given by:
\begin{equation}\label{eq:variationphi}
\begin{aligned}
  \delta \mathcal E_1(\varphi)(\delta\varphi)&=
 \int_0^{2\pi}\langle q_0-q_1\star \varphi ,\delta\varphi'. (q_1\star \varphi) - 2 (q_1\star \varphi)'. \delta\varphi \rangle_{\mathbb R^d}d\theta\\&\qquad\qquad+ 2\lambda \sum_{i=1}^n (\varphi(\theta_0^i)-\theta_1^i)\delta \varphi(\theta_0^i).
\end{aligned}
\end{equation}
The $L^2$--gradient of the energy functional \eqref{matching_funct3} is then given by:
\begin{equation*}
\begin{aligned}
 \operatorname{grad}(\mathcal E_1(\varphi))&= 
-\left\langle q_0,  \frac{(q_1\star\varphi)'}{\varphi'}\right\rangle_{\mathbb R^d}+ \left\langle q'_0, \frac{q_1\star\varphi}{\varphi'} \right\rangle_{\mathbb R^d} +2 \lambda \sum_{i=1}^n   (\varphi(\theta)-\theta_1^i) \delta_{\theta_0^i}(\theta),
\end{aligned}
\end{equation*}
where $\delta_{\theta_0^i}$ denotes the delta distribution and $q_j=R(c_j)$. 
\end{lemma}

\begin{proof}
Using the notation $q_j=R(c_j)$ the Energy functional can be written as
\begin{equation*}
\mathcal E_1(\varphi)= \|q_0- q_1\star\varphi\|^2_{L^2}+ \lambda \sum_{i=1}^n |\varphi(\theta_0^i)-\theta_1^i|^2\,.
\end{equation*}
We will calculate the variation of the two parts separately.
For the first part we have:
\begin{align*}
&\delta \left(\|q_0-q_1\star\varphi\|^2_{L^2}\right)(\delta\varphi)= \delta \left(\|q_0-\sqrt{\varphi'} (q_1\circ\varphi)\|^2_{L^2}\right)(\delta\varphi)\\&\qquad
=-2\int_0^{2\pi}\langle q_0-\sqrt{\varphi'} (q_1\circ\varphi),  \frac{\delta\varphi'}{2\sqrt{\varphi'}}q_1\circ\varphi + \sqrt{\varphi'}q_1'\circ\varphi\, \delta\varphi       \rangle_{\mathbb R^d}d\theta\\&\qquad=
-2\int_0^{2\pi}\langle q_0-q_1\star\varphi,  \frac{\delta\varphi'}{2\varphi'}q_1\star\varphi + (q_1'\star\varphi)\, \delta\varphi       \rangle_{\mathbb R^d}d\theta\,.
\end{align*}
To read off the $L^2$-gradient we have to integrate by parts the $\delta\varphi'$ term. Since $\delta \varphi$ vanishes at the boundary we have
\begin{align*}
&-2\int_0^{2\pi}\langle q_0- q_1\star\varphi,  \frac{\delta\varphi'}{2\varphi'}q_1\star\varphi \rangle_{\mathbb R^d}d\theta 
=2\int_0^{2\pi}\delta\varphi \left( \langle q_0-q_1\star\varphi,\frac{q_1\star\varphi}{2\varphi'} \rangle_{\mathbb R^d}\right)' d\theta\\
&\qquad=\int_0^{2\pi}\delta\varphi \langle q_0'-(q_1\star\varphi)',\frac{q_1\star\varphi}{\varphi'} \rangle_{\mathbb R^d} d\theta\\&\qquad\qquad+
\int_0^{2\pi}\delta\varphi \langle q_0-(q_1\star\varphi),\frac{(q_1\star\varphi)'\varphi'-(q_1\star\varphi)\varphi''}{\varphi'^2} \rangle_{\mathbb R^d} d\theta.
\end{align*}
Using that
\begin{equation}
q_1'\star\varphi= \frac{(q_1\star\varphi)'}{\varphi'}-\frac12(q_1\star\varphi)\frac{\varphi''}{(\varphi')^{2}}\,, 
\end{equation}
we obtain  the gradient of the first term:
\begin{align*}
\operatorname{grad}\left(  \|q_0- q_1\star \varphi\|^2_{L^2}     \right)
&= \left\langle q_0- q_1\star\varphi,  -2 q_1'\star\varphi+  \frac{(q_1\star\varphi)'}{\varphi'}-\frac{(q_1\star\varphi)\varphi''}{\varphi'^2}  \right\rangle_{\mathbb R^d}
\\&\qquad\qquad+ \left\langle q'_0- (q_1\star\varphi)', \frac{q_1\star\varphi}{\varphi'} \right\rangle_{\mathbb R^d}\\
& =  \left\langle q_0- q_1\star\varphi,  - \frac{(q_1\star\varphi)'}{\varphi'}\right\rangle_{\mathbb R^d}+ \left\langle q'_0- (q_1\star\varphi)', \frac{q_1\star\varphi}{\varphi'} \right\rangle_{\mathbb R^d}\\
&= -\left\langle q_0,  \frac{(q_1\star\varphi)'}{\varphi'}\right\rangle_{\mathbb R^d}+ \left\langle q'_0, \frac{q_1\star\varphi}{\varphi'} \right\rangle_{\mathbb R^d}\,.
\end{align*}
For the second summand we calculate
\begin{align*}
\delta \left(\sum_{i=1}^n |\varphi(\theta_0^i)-\theta_1^i|^2\right)(\delta\varphi) &= 2\sum_{i=1}^n (\varphi(\theta_0^i)-\theta_1^i)\delta \varphi(\theta_0^i)\\&= 
2\sum_{i=1}^n \int_0^{2\pi} \delta_{\theta_0^i}(\theta) (\varphi(\theta)-\theta_1^i) \delta \varphi(\theta) d\theta\,.
\end{align*}
Putting everything together, the formula for the gradient follows.
\end{proof}

\begin{remark}
 Since $\mathcal E_1^{\operatorname{cl}}$ can be seen as the restriction of $\mathcal E_1^{\operatorname{op}}$ to a co-dimension $d$ submanifold, 
the gradient of $\mathcal E_1^{\operatorname{cl}}$ is simply given by the projection onto the tangent space of this submanifold, cf.~Theorem \ref{SRVT-closed}.
\end{remark}

Using the above formulas, the implementation of the gradient descend algorithm is straightforward.
In \cite{Jermyn2011}, however, it has been shown that it can be beneficial to represent diffeomorphisms $\varphi$ as the tuple $\varphi = (x_0,\sqrt{\varphi'})$.
If one works with open curves the initial value $x_0$ is always zero. Then there is a one to one correspondence 
between $\varphi$ and $\sqrt{\varphi'}$. Denoting $\psi=\sqrt{\varphi'}$, the energy functional on open curves reads as:
\begin{equation*}
\mathcal E_2(\psi)= \|q_0-\psi. (q_1\circ(\int_0^\theta\psi^2d\tau))\|^2_{L^2}+ \lambda \sum_{i=1}^n |\int_0^{\theta_0^i}\psi^2d\tau-\theta_1^i|^2\,.
\end{equation*}
We can now also derive the variation of $\mathcal E$ in the  $\psi$--representation:
\begin{lemma}
The variation of $\mathcal E_2(\psi)$ in direction $\delta \psi$ is given by:
\begin{equation}\label{eq:variation_psi}
\begin{aligned}
&\delta \mathcal E_2(\psi)(\delta\psi)\\
&=  - 2 \int_0^{2\pi}\left( \delta \psi \left\langle q_0- q_1\star\varphi,   \frac{q_1 \star \varphi}{\psi}\right\rangle_{\mathbb R^d}
+ \left(\int_0^\theta\psi \delta\psi d\tau\right) \left\langle q_0- q_1\star\varphi, 2 q'_1\star\varphi           \right\rangle_{\mathbb R^d}\right)d\theta\\
&\qquad + 4\lambda \sum_{i=1}^n (\varphi(\theta_0^i)-\theta_1^i) \int_0^{\theta_0^i} \psi \delta \psi\, d\tau
\end{aligned}
\end{equation}
\end{lemma}
\begin{proof}
For a variation $\delta\psi$ we calculate
\begin{align*}
& \delta \Bigl(\Bigl\lVert q_0-\psi. \Bigl(q_1\circ\Bigl(\int_0^\theta\psi^2d\tau\Bigr)\Bigr)\Bigr\rVert^2_{L^2}\Bigr)(\delta\psi)\\&
=2\int_0^{2\pi}\Bigl\langle q_0-\psi. q_1\circ\Bigr(\int_0^\theta\psi^2d\tau\Bigr),  -\delta\psi . q_1\circ\Bigl(\int_0^\theta\psi^2d\tau\Bigr)\\
&\hskip150pt -2 \psi.\Bigl(q'_1\circ\Bigl(\int_0^\theta\psi^2d\tau\Bigr)\Bigr)  \int_0^\theta\psi \delta\psi\, d\tau \Bigr\rangle_{\mathbb R^d}d\theta\\
&= - 2\int_0^{2\pi}\delta \psi \Bigl\langle q_0-\psi. q_1\circ\Bigl(\int_0^\theta\psi^2d\tau\Bigr),   q_1\circ\Bigl(\int_0^\theta\psi^2d\tau\Bigr)\Bigr\rangle_{\mathbb R^d}d\theta\\
&\qquad- 4\int_0^{2\pi}\int_0^\theta\psi \delta\psi\, d\tau \Bigl\langle q_0-\psi. q_1\circ\Bigl(\int_0^\theta\psi^2d\tau\Bigr), \psi.\Bigl(q'_1\circ\Bigl(\int_0^\theta\psi^2d\tau\Bigr)\Bigr) \Bigr\rangle_{\mathbb R^d}d\theta\,.
\end{align*}
Using that $\int_0^\theta \psi^2d\tau=\varphi$ we obtain the formula for the first part.
For the second summand we calculate
\begin{align*}
\delta \left(\sum_{i=1}^n \Bigl|\int_0^{\theta_0^i}\psi^2d\tau-\theta_1^i\Bigr|^2\right)(\delta\psi) &=
2\sum_{i=1}^n \Bigl(\int_0^{\theta_0^i}\psi^2d\tau-\theta_1^i\Bigr) \int_0^{\theta_0^i}2\psi \delta \psi \, d\tau\,.
\end{align*}
\end{proof}

\begin{remark}\label{rem:otherFM}
  We do note that a gradient descent based algorithm can also be
  applied if the feature matching term $\operatorname{FM}$ is not the
  squared $\ell^2$-norm but rather a general differentiable function,
  the only difference being a corresponding modification of the last
  terms in the variations \eqref{eq:variationphi} and
  \eqref{eq:variation_psi}, respectively.
  For instance, in the case of a feature matching term
  $\operatorname{\widehat{FM}}$ defined on the parameter space $M^n$, the
  last term in~\eqref{eq:variation_psi} becomes
  \[
  2\lambda \sum_{i=1}^n \partial_i \operatorname{\widehat{FM}}(
  \varphi(\vtheta_0)-\vtheta_1) \int_0^{\theta_0^i} \psi \delta \psi\, d\tau.
  \]
\end{remark}

\subsection{Dynamic Programming}\label{dynamic}
As an alternative to the gradient descent method discussed above,
Dynamic Programming (DP) is often used to determine a piecewise linear approximation
of the optimal parametrization. We begin by introducing a local version of the energy functional \eqref{matching_funct3}.

Let $\mathcal{I} = \{\tau_0,\ldots,\tau_M\}$ be a discretization of the interval $[0,2\pi]$.
Given $k < i \in \mathcal{I}$ and $l < j \in \mathcal{I}$ and a strictly
increasing function $\varphi$ satisfying $\varphi([k,i]) = [l,j]$, we define
\begin{align}\label{Eq:LocalEnergy}
\bar{\mathcal E_1}(\varphi; k, l; i, j):= \int_k^i \Bigl\lvert\frac{c_0'}{\sqrt{|c_0'|}}-\sqrt{\varphi'}\frac{c_1'\circ\varphi }{\sqrt{|c_1'|} \circ \varphi}\Bigr\rvert^2_{\mathbb R^d}d\theta 
+ \lambda \sum_{m: l < \theta_1^m \le j} |\varphi(\theta_0^m)-\theta_1^m|^2.
\end{align}
In the special case where $\varphi$ is the linear function
\[
\varphi_\ind(\tau) := l  + (\tau - k)\frac{j-l}{i-k},
\]
and
\[
q_\ind(\tau) := \frac{c_1'\circ\varphi_{k,l;i,j}}{\sqrt{\lvert c'_1\circ\varphi_{k,l;i,j}\rvert}}\sqrt{\frac{j-l}{i-k}},
\]
is the corresponding SRV transform of the reparametrized curve,
the energy functional reduces to
\begin{equation}\label{eq:Eklij}
E(k,l;i,j) :=
\bar{\mathcal{E}_1}(\varphi_\ind; k, l; i,j)
=
\int_k^i | q_0 - q_\ind |^2_{\R^d} \, d\theta
+ \lambda\!\!\! \sum_{m: l < \theta_1^m \leq j} \lvert \varphi_\ind(\theta_0^m) - \theta_1^m\rvert^2.
\end{equation}

Denote now by $\Phi$ the set of all piecewise linear and increasing homeomorphisms \mbox{$\varphi\colon [0,2\pi] \to [0,2\pi]$} with vertices on the grid $\mathcal{I}\times \mathcal{I}$.
Denote moreover by $\Phi_{k,l}$ the set of all $\varphi\in\Phi$ satisfying $\varphi(k)=l$.
Now let for $i$, $j\in\mathcal{I}$
\[
H(i,j) := \min_{\varphi\in\Phi_{i,j}} \bar{\mathcal{E}_1}(\varphi; 0, 0; i,j)
\]
and denote by $\varphi_{i,j}$ the (any) corresponding minimizer.
That is, $H(i,j)$ is the minimal energy required for matching the curve segments
$c_0|_{[0,i]}$ and $c_1|_{[0,j]}$ using a piecewise linear reparametrization defined
on the given grid.
In order to find a global reparametrization, we need to find $H(2\pi,2\pi)$
and a corresponding optimal reparametrization $\bar{\varphi} := \varphi_{2\pi,2\pi}$.

Now note that $H$ satisfies the recursion
\begin{equation}\label{eq:Hij}
H(i,j) = \min_{k,l\in\mathcal{I},\ k < i,\ l < j} E(k,l;i,j) + H(k,l),
\end{equation}
because of the additivity of $\bar{\mathcal{E}_1}$.
Thus $\varphi_{i,j}$ is given by
\begin{equation}\label{eq:phiij}
\varphi_{i,j}(\tau) =
\begin{cases}
\varphi_\ind(\tau) & \tau \in [k,i],\\
\varphi_{k,l}(\tau) & \tau \in [0,k],
\end{cases}
\text{ with } (k,l) \in \argmin_{k,l \in \mathcal{I},\ k < i,\ l < j} E(k,l;i,j) + H(k,l).
\end{equation}

In practice, this computation consists of two steps.
In a first step, we create the $M\times M$ matrix $H$ inductively
while keeping track of the minimizing indices $k$ and $l$ (see~\eqref{eq:Hij}).
In the second step, we determine the function $\bar{\varphi}$ by
backtracking the minimizing indices and using formula~\eqref{eq:phiij}.

In order to speed up the computation, we can restrict the set of admissible
indices in~\eqref{eq:Hij} and consider only indices $k$, $l$ close to $i$, $j$.
In practice, this corresponds to a restriction of the possible slopes of
the piecewise linear reparametrization $\bar{\varphi}$.
See Fig. \ref{Fig:DynamicProgrammingMask} for an example.

\begin{figure}
\center
\def\svgwidth{\columnwidth / 3}
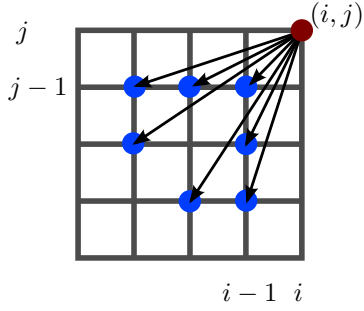
\caption{\label{Fig:DynamicProgrammingMask} The dynamic programming algorithm can be sped up massively by only considering predecessor indices $(k,l)$ close to the current index $(i,j)$ in \eqref{eq:Hij}.}
\end{figure}

\begin{remark}\label{rem:otherFM2}
  The minimization approach based on dynamic programming can also be
  applied for non-quadratic feature matching terms, as long as they
  decompose as
  \[
  \operatorname{FM}(\vC_0,\vC_1)
  = \sum_m \operatorname{FM}_m(C_0^m,C_1^m),
  \]
  with $\operatorname{FM}_i\colon \R^d \times \R^d \to \R_{\ge 0}$.
  For this, one only has to replace in~\eqref{eq:Eklij} the last sum
  by
  \[
  \sum_{m: l < \theta_1^m \leq j}
  \operatorname{FM}_m\bigl(c_0(\varphi_\ind(\theta_0^m)),c_1(\theta_1^m)\bigr).
  \]
  Note that this can be also used to implement hard constraints on the
  deviation of the feature points by setting
  \[
  \operatorname{FM}_m(C_0^m,C_1^m) =
  \begin{cases}
    +\infty & \text{ if } \lVert C_0^m-C_1^m\rVert > d_m,\\
    0 & \text{ if } \lVert C_0^m-C_1^m\rVert \le d_m,\\
  \end{cases}
  \]
  for some hard bounds $d_m \ge 0$.

  In addition, the dynamic programming approach readily extends to
  the symmetrization discussed in Section~\ref{sect:symmetric}, again
  as long as the feature matching terms decompose. Here, the last sum
  in~\eqref{eq:Eklij} has to be replaced by
  \[
  \sum_{m: l < \theta_1^m \leq j}
  \operatorname{FM}_m\bigl(c_0(\varphi_\ind(\theta_0^m)),c_1(\theta_1^m)\bigr)
  +
  \sum_{p: k < \theta_0^p \leq i}
  \operatorname{FM}_p\bigl(c_1(\varphi_\ind^{-1}(\theta_1^p)),c_0(\theta_0^p)\bigr).
  \]
  We stress here that the function $\varphi_\ind$ is linear on the
  interval $[k,i]$, and thus its inverse, which appears in the formula
  above, can be trivially computed.
\end{remark}

\section{Applications}\label{numerics}
As demonstrated in \cite{eslitzbichler2014}, shape matching techniques can be applied to certain computer animations such as, for instance, human walking motions.
This has uses in the entertainment industry (movie and TV production, and especially video games) as well as potential biomedical applications.

We will be working with skeletal animations, where motions are described in terms of bones and joints in an approximation of a human skeleton.
A typical approach to generate such animation data is to use \emph{motion capturing} methods, where a stuntman's motions are recorded by multiple cameras in a studio.
By tracking a multitude of points on the stuntman's body as he moves, the corresponding skeletal animation can be recovered.
These animations face a number of limitations, however, and often require additional postprocessing.
We refer to \cite{pejsa_state_2010, bruderlin_motion_1995, kovar_flexible_2003, kovar_automated_2004, eslitzbichler2014} for more details and examples.

A \emph{skeleton} is a directed acyclic graph where vertices and edges represent bones and joints, respectively.
A joint represents a transformation relationship between two bones.
In the case of human motions, transformations between bones are restricted to rotations.
Joints can have one to three degrees of freedom.
For example, the knee has one degree of freedom while the foot has two, and the shoulder has three.
Fig.~\ref{Fig:Skeleton} shows the skeleton used for our numerical experiment.
By representing rotations using Euler angles, we can collect all degrees of freedom of all joints in the skeleton as a high-dimensional torus, which we refer to as \emph{joint-space} and denote by \[\mathcal{J} := \mathbb{T}^d, \] where $d$ denotes the total number of degrees of freedom in the skeleton.

\begin{figure}
\center
\includegraphics[scale=0.25]{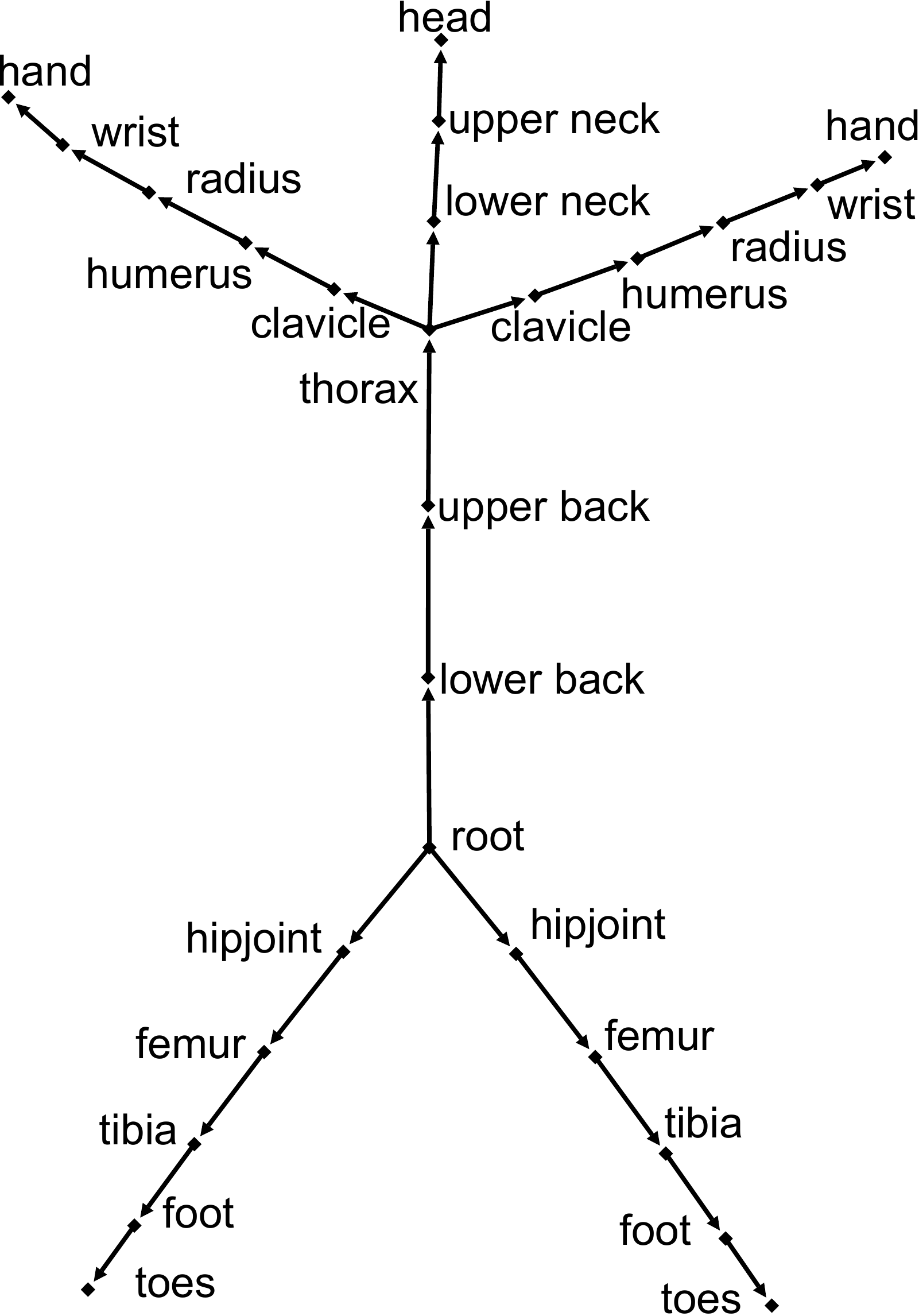} 
\caption{\label{Fig:Skeleton} This skeleton, which is based on data from the CMU Graphics Lab Motion Capture Database \cite{cmu2003}, was used for the animation experiments. Figure taken from \cite{eslitzbichler2014}.}
\end{figure}

An animation is then a function from a time interval to the joint-space $\mathcal{J}$, so that for every point in time we get a pose of the skeleton.

We can now unroll the joint-space torus in $\mathbb{R}^d$ and represent animations as parametrized curves $c_0$ and $c_1$ in this space.
The shape matching techniques developed previously can then be applied to these curves.
See \cite{eslitzbichler2014} for more details.

In Sect. \ref{Sect:FeatureAnimations}, we will show how feature matching can be successfully used to complement such existing shape matching methods.
We will begin, however, with a few examples of feature matching for planar curves to demonstrate the effects the additional feature term has on the curve matching.

\subsection{2d-curves}
As would be expected, adding feature points to shape matching can have a big impact on the resulting paths.

As a first example, we consider the matching of two open curves,
the first of which has three maxima and minima, while the second only has two
(see Fig.~\ref{Fig:Waves}).
Using only the elastic matching term without any specification of feature points,
the resulting minimum energy path between the two curves
approximately maps the first and the last extremum of the first curve to the first
and the last extremum of the second curve, while the extremum in the middle vanishes slowly.

Adding feature points, one can change the behavior of the optimal path significantly.
If, for instance, feature points are set on the last extremum of the first curve
but in the vicinity of only the last extremum of the second curve, then the optimal
path tries, during its evolution, to merge the last extrema of the first curve,
while its lower portion is matched quite closely to the lower half of the second curve
(see Fig.~\ref{Fig:Waves}, upper left).
Different behaviors follow from different choices of the feature points.

Fig.~\ref{Fig:ClosedCurves} shows similar behavior for closed curves.
When using a purely elastic matching term, moving from the first hand-pose to second one seen in Fig.~\ref{Fig:ClosedCurves}, we end up with visually unappealing interpolations.
By picking to corresponding fingertips on both hands as feature points, the algorithm achieves a much more natural looking transition from one shape to the other.
However, feature points need to be selected carefully, as the last row in Fig.~\ref{Fig:ClosedCurves} shows.
By attempting to match for example the thumb in one hand shape to the midway point between index and middle finger on the second hand shape, we induce a physically implausible interpolation that involves growing a new thumb.

We refer to the supplementary material\footnote{Supplementary material available at \url{https://wiki.math.ntnu.no/optimization/skeletal_animations}.} for videos demonstrating the differences. 
\begin{figure}
\center
\includegraphics[scale=1, trim= 2.2cm 1cm 0 1cm, clip=true]{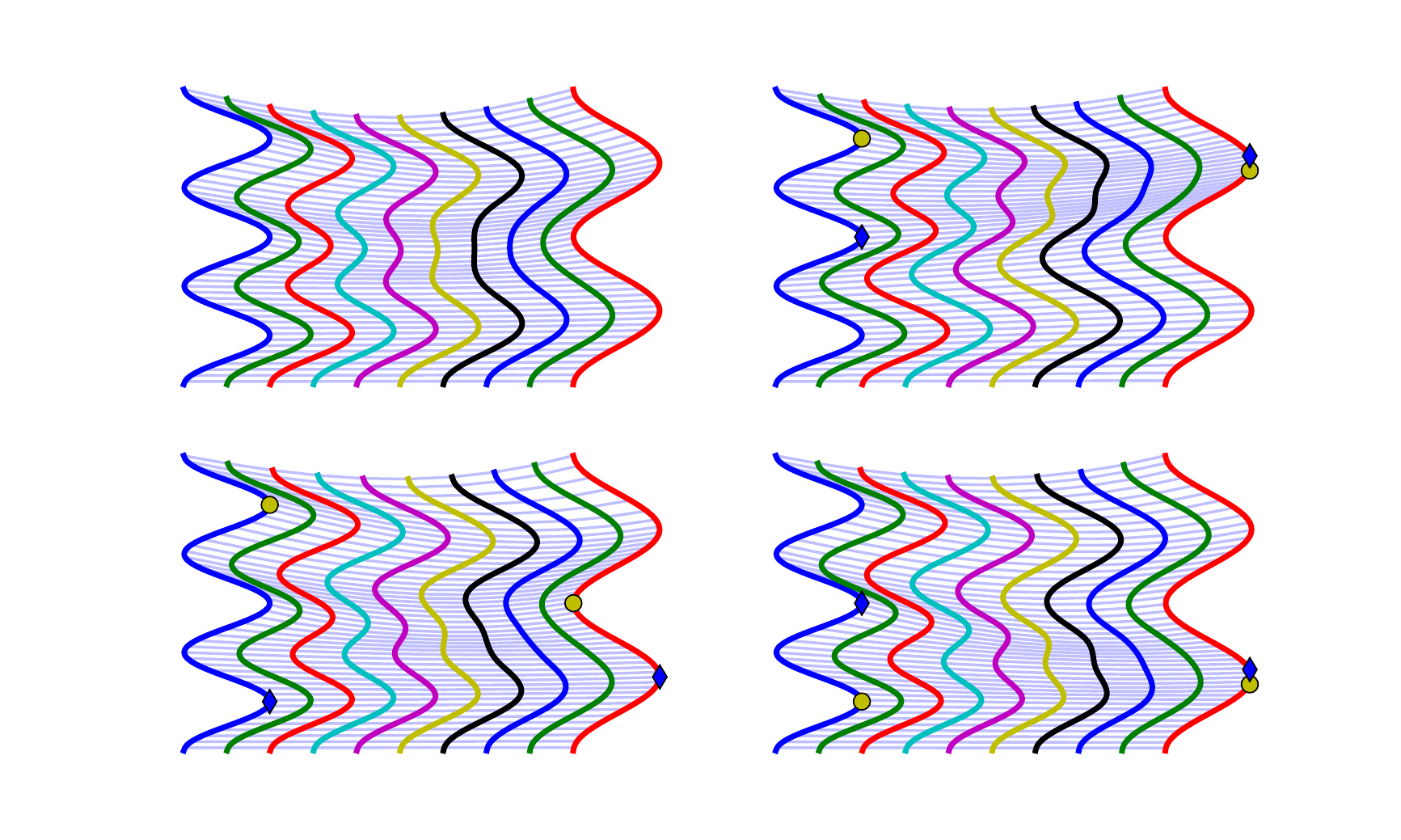} 
\caption{\label{Fig:Waves} Effect of picking different feature points when matching two shapes.
The top left figure shows results for shape matching using only an elastic energy functional without feature points.
The remaining figures show matching results for different combinations of feature points.
Corresponding markers on the left and right are matched, resulting in different paths between the given curves.}
\end{figure}

\begin{figure}
\center
\includegraphics[scale=1, trim= 2.2cm 1cm 0 1cm, clip=false]{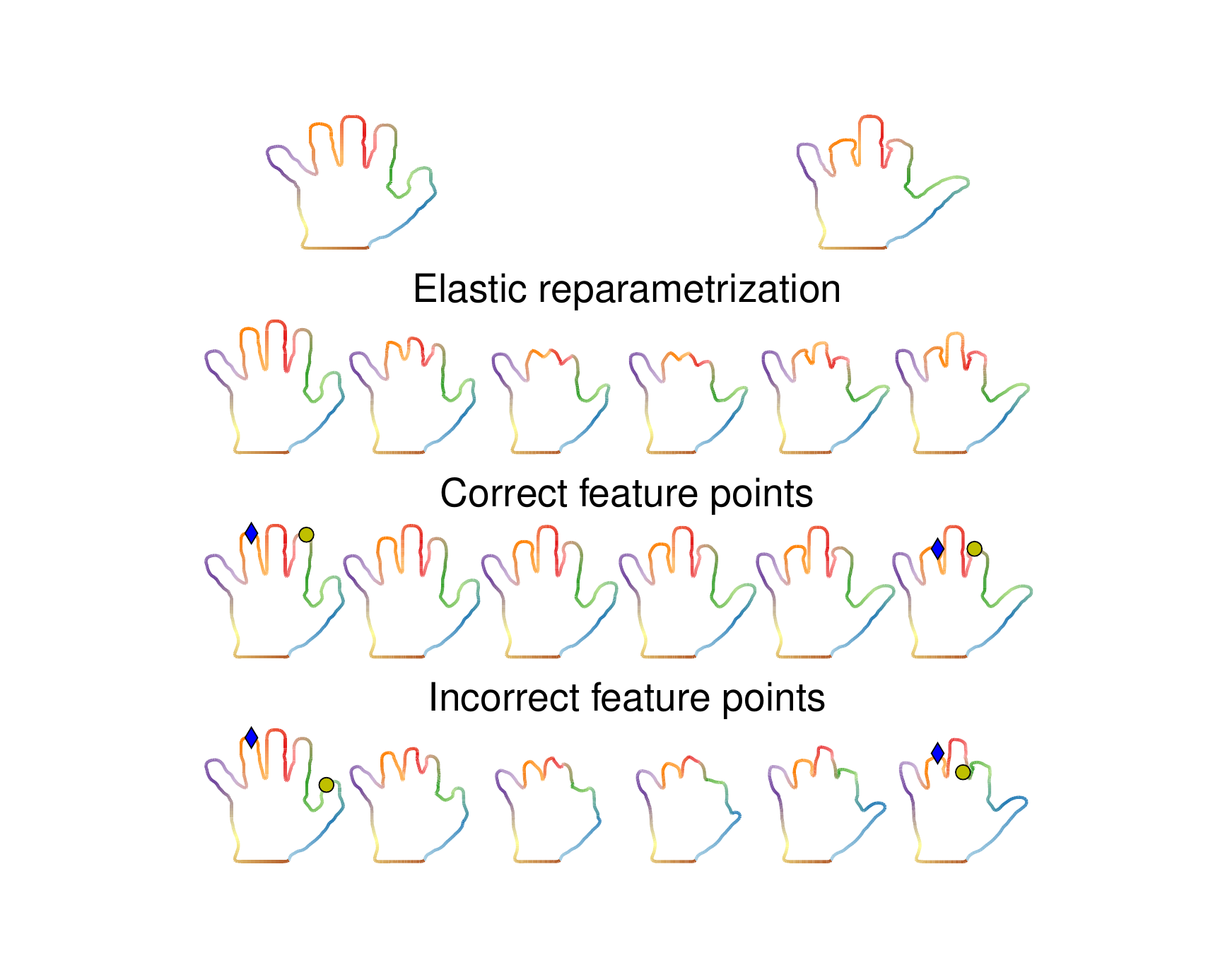} 
\caption{\label{Fig:ClosedCurves} Effect of picking different feature points when matching two different hand shapes.
In the top row, no feature points were set.
The purely elastic matching produces distorted shapes along the geodesic path between the two hand shapes.
In the middle row, feature points were set to match the tips of ring and index fingers correspondingly.
This results in more natural interpolated shapes.
In the bottom row, we see how incorrect feature matches cause some fingers to merge and new fingers to grow along the interpolation between the two shapes.
Corresponding markers on the left and right are matched, resulting in different paths between the given curves.
Colors along the curves indicate parametrization.
}
\end{figure}

\subsection{Applications to animations}
\label{Sect:FeatureAnimations}
We now turn to the use of feature point matching for animations.
Human animations come in an immensely large variety.
Walking motions alone can vary in speed, rhythm, step length, motions of the arms and so on.
Traditional elastic matching methods (i.e., without feature points) can be applied to a large number of animations, but can sometimes struggle with animations that have large differences in rhythm, for example when matching a walking animation to a limping animation.
Feature points can be used to help with determining an optimal reparametrization to align two animations in time.

\begin{figure}[h]
\center
\includegraphics[scale=1, trim= 2.5cm 1.5cm 0 1.5cm, clip=true]{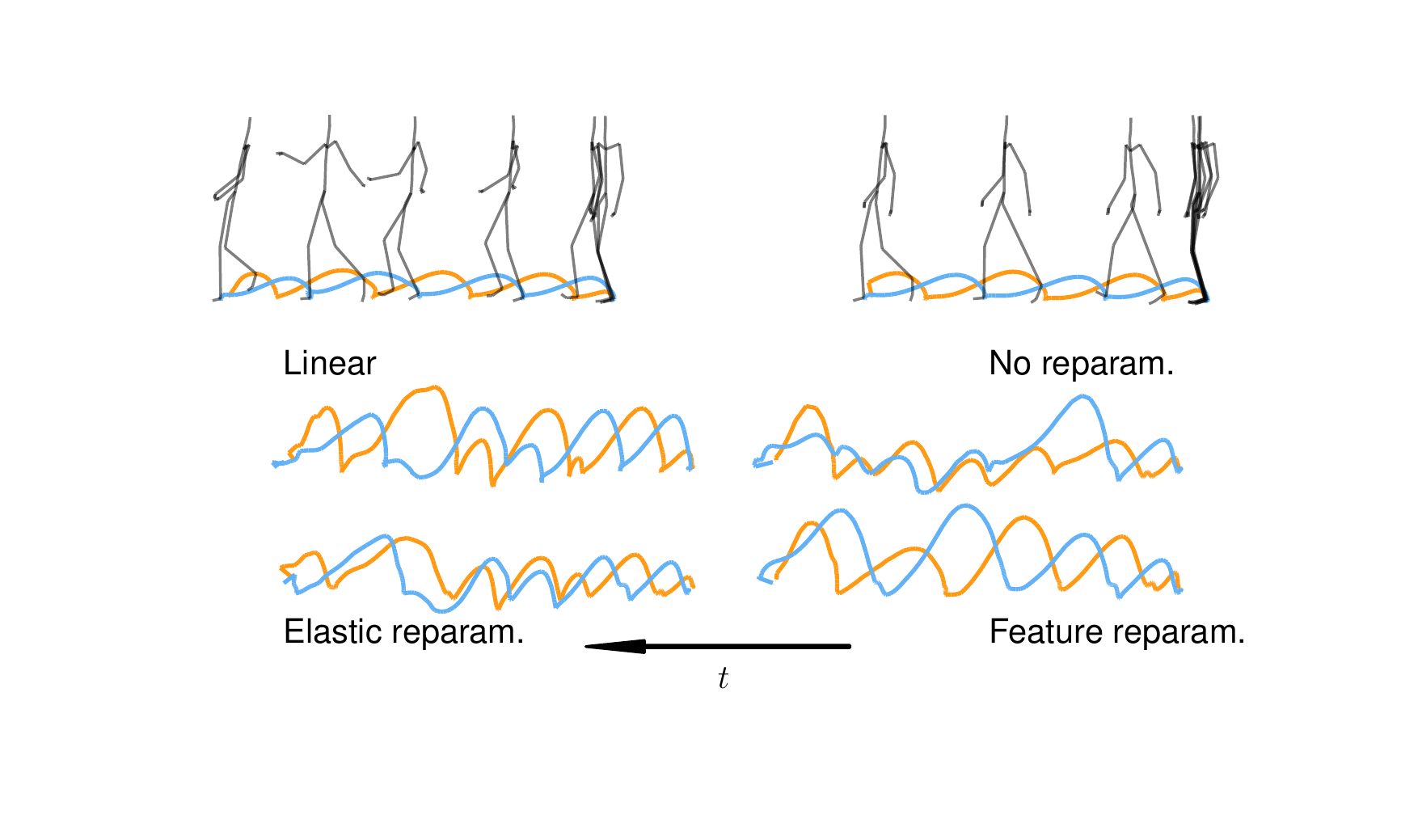}\vskip-15pt
\caption{\label{Fig:InterpolatedAnimations1} Example of using various methods to interpolate between two walking animations.
The blue and orange lines are the trajectories of the left and right feet respectively.
Note in particular how the two walking animations have different numbers of steps and how the various interpolated animations struggle with that.
We have from left to right and top to bottom the following methods: linear interpolation of the Euler angles, elastic matching without reparametrization, elastic matching with reparametrization and finally elastic and feature point matching with reparametrization.}
\end{figure}

\begin{figure}[h]
\center
\includegraphics[scale=1, trim= 2.5cm 1.5cm 0 1.5cm, clip=true]{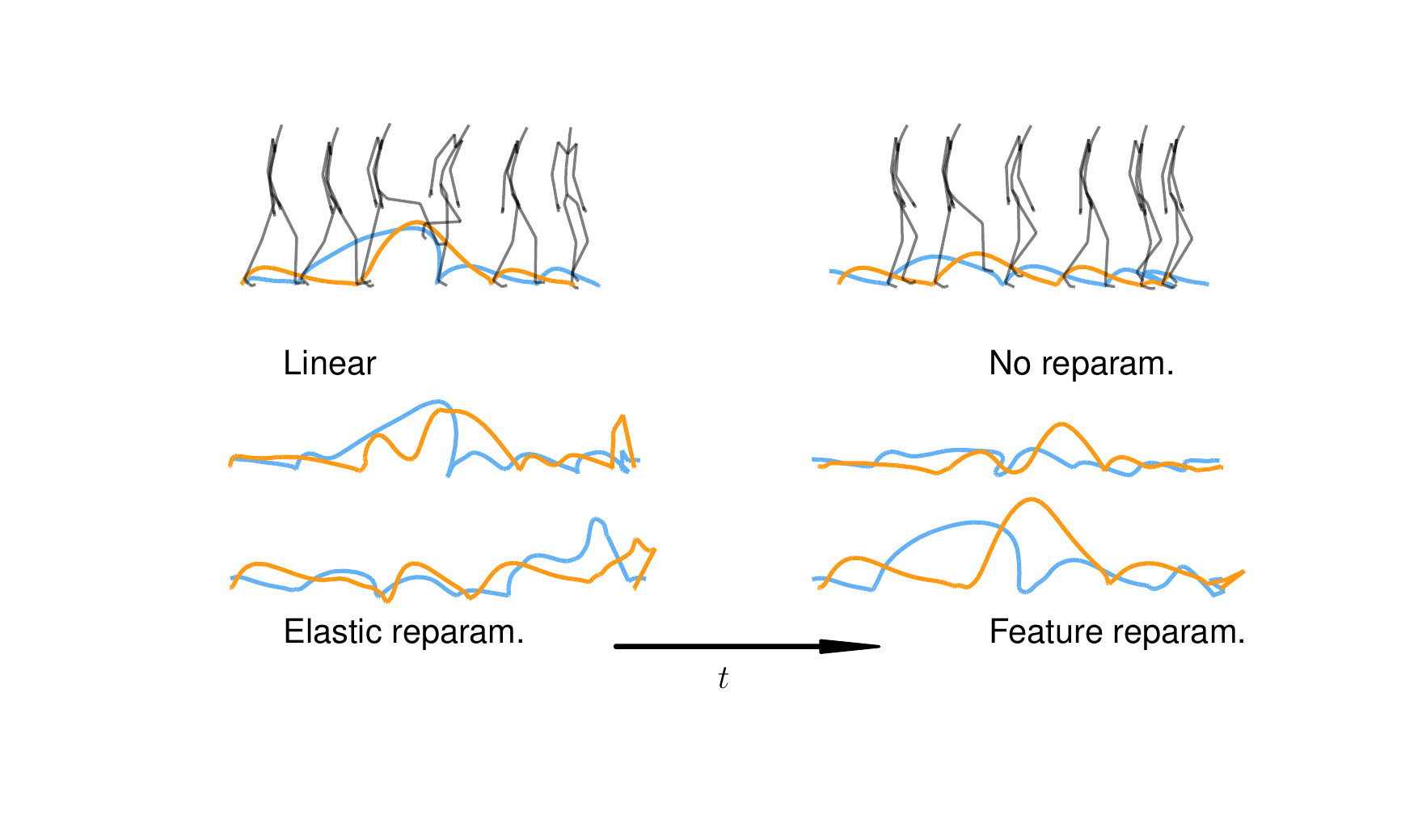}\vskip-15pt
\caption{\label{Fig:InterpolatedAnimations2} Example of using various methods to interpolate between two walking animations stepping over an obstacle.
The blue and orange lines are the trajectories of the left and right feet respectively.
We have from left to right and top to bottom the following methods: linear interpolation of the Euler angles, elastic matching without reparametrization, elastic matching with reparametrization and finally elastic and feature point matching with reparametrization.}
\end{figure}

Fig.~\ref{Fig:InterpolatedAnimations1} shows an example of using feature point information to aid in animation interpolation.
The goal is to calculate interpolations between two different walking animations.
These can be seen in the top row of Fig.~\ref{Fig:InterpolatedAnimations1}.
The animations differ in the number of steps, rhythm and arm motions.
In addition, the character in the second animation starts walking forward only after a short delay compared to the first.

The results of four different interpolation schemes (linear interpolation of the Euler angles, elastic matching with and without reparametrization and elastic and feature matching with reparametrization) can be seen on the bottom of Fig.~\ref{Fig:InterpolatedAnimations1}.
The two superimposed lines show the trajectories of the left and right feet for each calculated interpolation.
We can see how especially the varying numbers of steps in the two initial animation causes problems for the matching algorithms.

As feature points we picked the first three times when the left knee moves forward and past the right knee.
This is already enough information for the shape matching algorithm to determine a ``good'' (i.e., visually convincing) interpolation.

Similar results can be seen in the second example, which shows two walking animations stepping over an obstacle of different height as seen in Fig.~\ref{Fig:InterpolatedAnimations2}.
Again, the superimposed lines represent the trajectories of the feet, with noticeable bumps as the character steps over the obstacle.
As before, we have picked the first three times when the left knee moves forward and past the right knee as feature points.
With this additional information, the feature point matching algorithm manages to produce a visually convincing interpolation between the two animations, whereas the other matching algorithms fail and produce only garbled results.

We refer to the supplementary material\footnote{Supplementary material available at \url{https://wiki.math.ntnu.no/optimization/skeletal_animations}.} for a video demonstrating the differences.


\end{document}